\documentclass[12pt, a4paper]{article}

\usepackage[a4paper]{geometry}
\usepackage{pgfplots}
\pgfplotsset{
    x tick style={color=black},
    y tick style={color=black}
}

\usepackage{url}
\usepackage{listings}
\usepackage{amssymb}
\usepackage{graphicx}
\usepackage{algcompatible}
\usepackage{algorithm}
\usepackage{url}
\usepackage{rotating}
\usepackage{booktabs}
\usepackage{subfig}
\usepackage{amsmath}
\usepackage{bbm}

\renewcommand{\labelenumi}{(\alph{enumi})}
\renewcommand\theenumi\labelenumi

\usepackage[english]{babel}
\usepackage{lscape}

\usepackage[utf8]{inputenc}
\usepackage{xspace}
\usepackage{amsmath,amsthm,amssymb,mathtools}
\usepackage{lmodern}

\usepackage[algo2e,ruled,vlined,linesnumbered]{algorithm2e}

\usepackage{xcolor}
\usepackage{tikz}
\usepackage{graphicx}
\usepackage{booktabs} 
\usepackage{xspace}
\usepackage[algo2e,ruled,vlined,linesnumbered]{algorithm2e}
\usepackage{pdfpages}
\usepackage{multirow}
\usepackage{siunitx}
\usepackage{subfig}
\usepackage{algcompatible}
\usepackage{algorithm}
\usepackage{arydshln} 
\usepackage{ragged2e}

\renewcommand{\labelenumi}{\theenumi}
\renewcommand{\theenumi}{(\roman{enumi})}

\newtheorem{theorem}{Theorem}

\newcommand{\oea}{\mbox{$(1 + 1)$~EA}\xspace}

\newcommand{\onemax}{\textsc{OneMax}\xspace}
\newcommand{\LO}{\textsc{Leading\-Ones}\xspace}
\newcommand{\leadingones}{\LO}

\newcommand{\binval}{\textsc{BinVal}\xspace}

\newcommand{\jump}{\textsc{Jump}\xspace}
\newcommand{\DLB}{\textsc{DeceptiveLeadingBlocks}\xspace}

\DeclareMathOperator{\poly}{poly}

\newcommand{\R}{\ensuremath{\mathbb{R}}}

\newcommand{\Z}{\ensuremath{\mathbb{Z}}}

\newcommand{\calL}{\ensuremath{\mathcal{L}}}

\usepackage{hyperref}  

\begin{document}
{\sloppy
\title{From Understanding Genetic Drift to a Smart-Restart Parameter-less Compact Genetic Algorithm\thanks{Extended version of a paper~\cite{DoerrZ20gecco} that appeared in the proceedings of GECCO 2020. This version also discusses how the algorithms considered behave in the presence of additive posterior noise. Also, this version contains all mathematical proofs. Authors are given in alphabetic order. Both authors contributed equally to this work and both act as corresponding authors. }}

\author{Benjamin Doerr\\ Laboratoire d'Informatique (LIX)\\ \'Ecole Polytechnique, CNRS\\ Institut Polytechnique de Paris\\ Palaiseau, France
\and Weijie Zheng\\ Department of Computer Science and Engineering \\ Southern University of Science and Technology \\ Shenzhen, China\\ School of Computer Science and Technology\\ University of Science and Technology of China\\Hefei, China}



\maketitle
\begin{abstract}
One of the key difficulties in using estimation-of-distribution algorithms is choosing the population size(s) appropriately: Too small values lead to genetic drift, which can cause enormous difficulties. In the regime with no genetic drift, however, often the runtime is roughly proportional to the population size, which renders large population sizes inefficient. 
  
Based on a recent quantitative analysis which population sizes lead to genetic drift, we propose a parameter-less version of the compact genetic algorithm that automatically finds a suitable population size without spending too much time in situations unfavorable due to genetic drift. 

We prove a mathematical runtime guarantee for this algorithm and conduct an extensive experimental analysis on four classic benchmark problems both without and with additive centered Gaussian posterior noise. The former shows that under a natural assumption, our algorithm has a performance very similar to the one obtainable from the best problem-specific population size. The latter confirms that missing the right population size in the original cGA can be detrimental and that previous theory-based suggestions for the population size can be far away from the right values; it also shows that our algorithm as well as a previously proposed parameter-less variant of the cGA based on parallel runs avoid such pitfalls. Comparing the two parameter-less approaches, ours profits from its ability to abort runs which are likely to be stuck in a genetic drift situation.
\end{abstract}

%
%



\maketitle

\section{Introduction}

Estimation-of-distribution algorithms (EDAs)~\cite{LarranagaL02,PelikanHL15}  are a branch of evolutionary algorithms (EAs) that evolve a probabilistic model instead of a population. The update of the probabilistic model is based on the current model and the fitness of a population sampled according to the model. The size of this population is crucial for the performance of the EDA. Taking the univariate marginal distribution algorithm (UMDA)~\cite{MuhlenbeinP96} with artificial frequency margins $\{1/n, 1-1/n\}$ optimizing the $n$-dimensional \DLB problem as an example, Lehre and Nguyen~\cite[Theorem 4.9]{LehreN19foga} showed that if the population size is small ($\lambda=\Omega(\log n) \cap o(n)$) and the selective pressure is standard ($\mu/\lambda \ge 14/1000$), then the expected runtime is $\exp(\Omega(\lambda))$.
The essential reason for this weak performance, quantified in Doerr and Zheng~\cite{DoerrZ20tec} but observed also in many previous works, is that the small population size leads to strong genetic drift, that is, the random fluctuations of frequencies caused by the random sampling of search points eventually move some sampling frequencies towards a boundary of the frequency range that is not justified by the fitness. Doerr and Krejca's recent work~\cite{DoerrK20evocop} showed that when the population size is large enough, that is, $\lambda =\Omega(n\log n)$ and $\mu=\Theta(\lambda)$, the genetic drift effect is weak and with high probability, the UMDA finds the optimum in $\lambda(n/2+2e\ln n)$ fitness evaluations. This runtime bound is roughly proportional to the population size $\lambda$. Assuming that this bound describes the true runtime behavior (no lower bound was shown in~\cite{DoerrK20evocop}, but from the proofs given there this assumption appears realistic), we see that a too large population size will again reduce the efficiency of the algorithm. 

We refer to the recent survey of Krejca and Witt~\cite{KrejcaW20bookchapter} for more runtime analyses of EDAs. For most of the results presented there, a minimum population size is necessary and then the runtime is roughly proportional to the population size. In a word, for many EDAs a too small population size leads to genetic drift, while a too large size results in inefficiency. Choosing the appropriate population size is one of the key difficulties in the effective usage of EDAs. 

We note that there have been attempts to define EDAs that are not prone to genetic drift~\cite{FriedrichKK16,DoerrK20tec}, also with promising results, but from the few existing results (essentially only for the \onemax, \binval, and \leadingones benchmarks) it is hard to estimate how promising these ideas are in general, in particular, for more complex optimization problems. For this reason, in this works we rather discuss how to set the parameters for established EDAs.

Parameter tuning and parameter control have successfully been used to find suitable parameter values. 
However, both approaches will usually only design problem-specific strategies and often require sophisticated expertise to become successful. 
In order to free the practitioner from the task of choosing parameters, researchers have tried to remove the parameters from the algorithm while trying to maintain a good performance, ideally comparable to the one with best parameter choice for the problem to be solved. Such algorithms are called parameter-less\footnote{Not surprisingly, many mechanisms to remove parameters have themselves some parameters. The name \emph{parameter-less} might still be justified when these meta-parameters have a less critical influence on the performance of the algorithm.}. 

This paper will address the problem of designing a parameter-less compact genetic algorithm (cGA). In an early work, Harik and Lobo~\cite{HarikL99} proposed two strategies to remove the population size of crossover-based genetic algorithms. One basic strategy is doubling the population size and restarting when all individuals' genotypes have become identical. The drawback of this strategy is the long time it takes to reach the termination criterion once genetic drift has become detrimental. Harik and Lobo proposed a second strategy in which multiple populations with different sizes run simultaneously, smaller population sizes may use more function evaluations, but are removed once their fitness value falls behind the one of larger populations. Their experimental results showed that their genetic algorithm with this second strategy only had a small performance loss over the same genetic algorithm with optimal parameter settings. Many extensions of this strategy and applications with other optimization algorithms have followed, giving rise to the extended compact genetic algorithm~\cite{LimaL04}, the hierarchical Bayesian optimization algorithm~\cite{PelikanL04}, and many others. 

Goldman and Punch~\cite{GoldmanP14} proposed the parameter-less population pyramid, called P3, to iteratively construct a collection of populations. In P3, the population in the pyramid expands iteratively by first adding a currently not existing solution obtained by some local search strategy into the lowest population, and then utilizing some model-building methods to expand the population in all hierarchies of the pyramid. Since initially no population exists in the pyramid, this algorithm frees the practitioner from specifying a population size. For EDAs, Doerr~\cite{Doerr21cgajump} recently proposed another strategy building a parallel EDA running with exponentially growing population size. With a careful strategy to assign the computational resources, he obtained that under a suitable assumption this parallel EDA only had a logarithmic factor performance loss over the corresponding original EDA using the optimal population size. 


\textbf{Our contribution:}
The above parameter-less strategies use clever but indirect ways to handle the possibly long wasted time caused by genetic drift. In this work, we aim at a more direct approach by exploiting a recent mathematical analysis which predicts when genetic drift arises. Doerr and Zheng~\cite{DoerrZ20tec} have theoretically analyzed the boundary hitting time caused by genetic drift. In very simple words, their result indicates that genetic drift in a bit position of the compact genetic algorithm (cGA) occurs when the runtime of the algorithm exceeds $4\mu^2$, where $\mu$ is the hypothetical population size of the cGA. We use this insight to design the following parameter-less version of the cGA. Our parameter-less cGA, called \emph{smart-restart cGA},\footnote{The authors are thankful to an anonymous reviewer of~\cite{DoerrZ20gecco} for suggesting this name.} is a simple restart process with exponentially growing population size. It stops a run once the risk of genetic drift is deemed too high, based on the analysis in~\cite{DoerrZ20tec}. 

Since Doerr and Zheng~\cite{DoerrZ20tec} proved that a neutral frequency reaches the boundaries of the frequency range in an expected number of $4\mu^2$ generations, it is natural to set $B=b\mu^2$ with $b=O(1)$ as the generation budget for a run with population size~$\mu$. This builds on the observation that Markov's inequality implies that with probability at least $1-4/b$, a boundary is reached in $b\mu^2$ generations. Since genetic drift affects neutral bits stronger than those subject to a clear fitness signal, we can pessimistically take $b\mu^2$ generations as the termination budget for a cGA run with population size $\mu$. 

Note that we do not restrict $b$ to be a constant. The reasoning above stems from considering a single frequency only. Since there are $n$ frequencies, one may speculate that the first of these reaches a boundary already in $\Theta(\mu^2/\ln n)$ generations. We do not have a fully rigorous analysis showing that the first of the frequencies reaches a boundary in $O(\mu^2/\ln n)$ iterations, but the tail bound in~\cite{DoerrZ20tec} shows that this does not happen earlier and our experiments suggest that taking this smaller budget is indeed often profitable. Hence, it makes sense to allow $b=o(1)$ and we shall in particular regard the setting $b = \Theta(1/\ln n)$.

When the generation budget $B=b\mu^2$ runs out before the optimum is found, we restart the cGA with population size $\mu:=U\mu$ for the update factor $U>1$.

For our algorithm, we prove a mathematical runtime guarantee. We assume that there are numbers $\tilde \mu$ and $T$ such that the cGA with all population size $\mu \ge \tilde{\mu}$ solves the given problem in time $\mu T$ with probability $p > 1 - \frac 1 {U^2}$. Such a runtime behavior is indeed often observed, see, e.g.,~\cite{KrejcaW20bookchapter}.
%
We theoretically prove that under this assumption, our smart-restart cGA with population size update factor $U$ and generation budget factor $b$ solves the problem in expected time $\left(\frac{U^2}{U^2-1}+\frac{(1-p)U^2}{1-(1-p)U^2}\right)\max\left\{b\tilde{\mu}^2,\frac{T^2}{b}\right\}+\frac{pU}{1-(1-p)U}\tilde{\mu}T$, which is $O(\max\{b \tilde \mu^2, \frac{T^2}{b}, \tilde \mu T\})$ when treating $U$ and $p$ as constants.

Together with the known results that the cGA with all $\mu = \Omega(\sqrt n \log n) \cap O(\poly(n))$ optimizes \onemax and \jump functions with jump size $k\le \frac{1}{20} \ln n -1$ in time $O(\sqrt n \mu)$ with probability $1-o(1)$~\cite{SudholtW19,Doerr21cgajump}, our general runtime result implies that our algorithm with generation budget factor $b=\Theta(1/\ln n)$ optimizes these functions (except for a rare event of probability at most $n^{-\omega(1)}$) in expected time $O(n \log n)$, which is the asymptotically best performance the cGA can have with an optimal choice of $\mu$. In a similar manner, we show that the smart-restart cGA optimizes \onemax in the presence of posterior noise essentially as fast as shown for the original cGA with optimal population size~\cite{FriedrichKKS17}.

We also conduct an extensive experimental analysis of the original cGA, the parallel-run parameter-less cGA and our smart-restart parameter-less cGA on the \onemax, \LO, \jump, and \DLB functions both in the absence of noise and with additive centered Gaussian posterior noise as considered in~\cite{FriedrichKKS17}. For the original cGA and the noiseless scenario, it confirms, for the first time experimentally for these benchmarks, that small population sizes can be detrimental and that from a certain population size on, a roughly linear increase of the runtime can be observed. It also confirms experimentally the insight of the (asymptotic) theory result that, with the right population size, the cGA can be very efficient on \jump functions. For example, we measure a median runtime of $4 \cdot 10^6$ on the \jump function with $n=50$ and $k=10$, parameters for which, e.g., the classic \oea would take more than $10^{17}$ iterations. In the noisy settings, we observe that the population sizes suggested (for \onemax) by the theoretical analysis~\cite{FriedrichKKS17} are much higher (roughly by a factor of $1{,}000$) than what is really necessary, leading to runtime increases of similar orders of magnitudes.

The two parameter-less versions of the cGA generally perform very well, leading to runtimes that are only mildly above those stemming from the best problem-specific parameter values. Overall, the smart-restart cGA is faster, which suggests that the design concept of saving time by aborting unprofitable runs has worked out. 

The remainder of this paper is structured as follows. Section~\ref{sec:pre} introduces the preliminaries including a detailed description of the compact genetic algorithm, the parallel-run cGA, and the additive centered Gaussian noise environment. The newly-proposed smart-restart cGA will be stated in Section~\ref{sec:parameterlesscGA}. Sections~\ref{sec:theory} and~\ref{sec:exper} show our theoretical results and experimental analyses respectively. Section~\ref{sec:con} concludes our paper.

\section{Preliminaries}
\label{sec:pre}

In this paper, we consider algorithms maximizing pseudo-boolean functions $f:\{0,1\}^n \rightarrow \R$. Since our smart-restart cGA is based on the original cGA of Harik, Lobo, and Goldberg~\cite{HarikLG99} and since we will compare our algorithm with the parallel-run cGA~\cite{Doerr21cgajump}, this section will give a brief introduction to these algorithms.

\subsection{The Compact Genetic Algorithm}\label{subsec:cga}
The compact genetic algorithm (cGA) with hypothetical population size $\mu$ samples two individuals in each generation and moves the sampling frequencies by an absolute value of $1/\mu$ towards the bit values of the better individual. Usually, in order to avoid frequencies reaching the absorbing boundaries $0$ or $1$, the artificial margins $1/n$ and $1-1/n$ are utilized, that is, we restrict the frequency values to be in the interval $[1/n,1-1/n]$. The following Algorithm~\ref{alg:cGA} shows the details. As common in runtime analysis, we do not specify a termination criterion. When talking about the runtime of an algorithm, we mean the first time (measured by the number of fitness evaluations) an optimum was sampled.
\begin{algorithm}[!ht]
\caption{The cGA to maximize a function $f: \{0,1\}^n \rightarrow \R$ with hypothetical population size $\mu$}
{\small
 \begin{algorithmic}[1]
 \STATE{$p^0=(\tfrac{1}{2}, \tfrac{1}{2},\dots,\tfrac{1}{2})\in [0,1]^n$}
 \FOR {$g=1,2,\dots$}
 \STATEx {$\quad\%\%$\textsl{Sample two individuals $X_1^g,X_2^g$}}
 \FOR {$i=1,2$}
 \FOR {$j=1,2,\dots,n$}
 \STATE $X_{i,j}^g \leftarrow 1$ with probability $p_{j}^{g-1}$ and $X_{i,j}^g \leftarrow 0$ with probability $1-p_{j}^{g-1}$.
 \ENDFOR
 \ENDFOR
 \STATEx {$\quad\%\%$\textsl{Update of the frequency vector}}
 \IF{$f(X_1^g) \ge f(X_2^g)$}
 \STATE {$p'=p^{g-1}+\tfrac{1}{\mu}(X_1^g-X_2^g)$};
 \ELSE 
  \STATE {$p'=p^{g-1}+\tfrac{1}{\mu}(X_2^g-X_1^g)$};
  \ENDIF
 \STATE {$p^g=\min \{\max\{\tfrac{1}{n},p'\},1-\tfrac{1}{n}\}$};
 \ENDFOR
 \end{algorithmic}
 \label{alg:cGA}
}
\end{algorithm}

\subsection{The Parallel-run cGA}
\label{subsec:prllcGA}
The parallel EDA framework was proposed by Doerr~\cite{Doerr21cgajump} as a side result when discussing the connection between runtime bounds that hold with high probability and the expected runtime. For the cGA, this framework yields the following \emph{parallel-run cGA}. In the initial round $\ell = 1$, we start process $\ell=1$ to run the cGA with population size $\mu=2^{\ell -1}$ for $1$ generation. In round $\ell = 2, 3, \dots$, all running processes $j=1,\dots, \ell-1$ run $2^{\ell-1}$ generations and then we start process $\ell$ to run the cGA with population size $\mu=2^{\ell-1}$ for $\sum_{i=0}^{\ell-1} 2^i$ generations. The algorithm terminates once any process has found the optimum. 
Algorithm~\ref{alg:prllcGA} shows the details of the parallel-run cGA.

Based on the following assumption, Doerr~\cite{Doerr21cgajump} proved that the expected runtime for this parallel-run cGA is at most $6 \tilde \mu T (\log_2(\tilde \mu T) + 3)$.

\textbf{Assumption~\cite{Doerr21cgajump}:} Consider using the cGA with population size $\mu$ to maximize a given function $f$. Assume that there are unknown $\tilde{\mu}$ and $T$ such that the cGA for all population sizes $\mu \ge \tilde{\mu}$ optimizes this function $f$ in $\mu T$ fitness evaluations with probability at least $\tfrac 34$.\\

\begin{algorithm}[!ht]
\caption{The parallel-run cGA to maximize a function $f: \{0,1\}^n \rightarrow \R$}
{\small
 \begin{algorithmic}[1]
 \STATE Process 1 runs cGA (Algorithm~\ref{alg:cGA}) with population size $\mu=1$ for $1$ generation.
 \FOR {round $\ell=2,\dots$}
 \STATE Processes $1,\dots, \ell-1$ continue to run for another $2^{\ell-1}$ generations, one process after the other one.
 \STATE Start process $\ell$ to run cGA (Algorithm~\ref{alg:cGA}) with population size $\mu=2^{\ell-1}$ and run it for $\sum_{i=0}^{\ell-1}2^i$ generations. 
 \ENDFOR
 \end{algorithmic}
 \label{alg:prllcGA}
}
\end{algorithm}


\subsection{Additive Centered Gaussian Posterior Noise}\label{ssec:intronoise}

In practical applications, one often encounters various forms of uncertainty. One of these is a noisy access to the objective function. Friedrich, K\"otzing, Krejca, and Sutton~\cite{FriedrichKKS17} analyzed how the cGA optimizes the \onemax problem under additive centered Gaussian posterior noise. They proved that for all noise intensities (variances $\sigma^2$ of the Gaussian distribution), there is a population size $\mu = \mu(\sigma^2)$ which depends only polynomially on $\sigma^2$ (that is, $\mu(\sigma^2)$ is a polynomial in $\sigma^2$) so that the cGA with this population size efficiently solves the \onemax problem. This was called \emph{graceful scaling}. They also provided a restart scheme that obtains this performance without knowledge of the noise intensity (however, it requires to know the polynomial $\mu(\sigma^2)$). Hence these results show that the cGA can deal well with the type of noise regarded, and much better than many classic evolutionary algorithms (see the lower bounds in~\cite{GiessenK16,FriedrichKKS17}), but this still needs an action by the algorithm user, namely an appropriate choice of the population size $\mu$.

As we shall show in this work, our restart scheme is also able to optimize noisy versions of \onemax and many other problems, but without knowing the polynomial $\mu(\sigma^2)$ and using significantly more efficient values for the population size. For \onemax, we prove rigorously that we obtain essentially the performance of the original cGA with best choice of the population size (Theorem~\ref{thm:noisyom}), where we profit from the fact that the runtime analysis of~\cite{FriedrichKKS17} shows that the cGA also for noisy \onemax functions essentially satisfies our main assumption that from a certain population size on, the runtime of the cGA is at most proportional to the population size. 

We conduct experiments for various benchmark functions in this noise model. They indicate that also for problems different from \onemax, the graceful scaling property holds. However, they also show that much smaller population sizes suffice to cope with the noise. Consequently, our smart-restart cGA (as well as the parallel-run cGA from~\cite{Doerr21cgajump}) optimizes \onemax much faster than the algorithms proposed in~\cite{FriedrichKKS17}. This is natural since the parameter-less approaches also try smaller (more efficient in case of success) population sizes, whereas the approaches in~\cite{FriedrichKKS17} use a population size large enough that one can prove via mathematical means that they will be successful with high probability. 

We now make precise the additive centered Gaussian noise model. We take the common assumption that whenever the noisy fitness of a search point is regarded in a run of the algorithm, its noisy fitness is computed anew, that is, with newly sampled noise. This avoids that a single exceptional noise event misguides the algorithm for the remaining run. A comparison of the results in~\cite{SudholtT12} (without independent reevaluations) and~\cite{DoerrHK12ants} (with reevaluations) shows how detrimental sticking to previous evaluations can be. We regard posterior noise, that is, the noisy fitness value is obtained from a perturbation of the original fitness value (independent of the argument) as opposed to anterior noise, where the algorithm works with the fitness of a perturbed search point. We regard additive perturbations, hence the perceived fitness of a search point $x$ is $f(x)+D$, where $f$ is the original fitness function and $D$ is an independent sample from a distribution describing the noise. Since we consider centered Gaussian noise, we always have $D \sim \mathcal{N}(0,\sigma^2)$, where $\mathcal{N}(0,\sigma^2)$ denotes the Gaussian distribution with expectation zero and variance $\sigma^2 \ge 0$. Obviously, the classic noise-free optimization scenario is subsumed by the special case $\sigma^2=0$.

\section{The Smart-Restart {cGA}}
\label{sec:parameterlesscGA}

In this section, we introduce our parameter-less cGA, called \emph{smart-restart cGA}. In contrast to the parallel-run cGA it does not run processes in parallel, which is an advantage from the implementation point of view. The main advantage we aim for is that by predicting when runs become hopeless, we can abort these runs and save runtime.

To detect such a hopeless situation, we use the first tight quantification of the genetic drift effect of the EDAs by Doerr and Zheng~\cite{DoerrZ20tec}. Detailedly, they proved that in a run of the cGA with hypothetical population size $\mu$ a frequency of a neutral bit will reach the boundaries of the frequency range in expected number of at most $4\mu^2$ generations, which is asymptotically tight. By Markov's inequality the probability that a boundary is reached in $b\mu^2$, $b>4$, generations is at least $1-4/b$. 

This suggests the restart scheme described in Algorithm~\ref{alg:nonpcGA}. We start with a small population size of $\mu = 2$. We then repeat running the cGA with population size $\mu_{\ell}=2U^{\ell-1}$ for $B_{\ell}=b\mu_{\ell}^2$ generations. We call $U$ the \emph{update factor} for the population size and $b$ the \emph{generation budget factor}.
 As before, we do not specify a termination criterion since for our analysis we just count the number of fitness evaluations until a desired solution is found.

\begin{algorithm}[!ht]
\caption{The smart-restart cGA to maximize a function $f: \{0,1\}^n \rightarrow \R$ with update factor $U$ and generation budget factor $b$. }
{\small
 \begin{algorithmic}[1]
 \FOR {round $\ell=1,2,\dots$}
 \STATE Run the cGA (Algorithm~\ref{alg:cGA}) with population size $\mu_{\ell}=2U^{\ell-1}$ for $B_{\ell}=b\mu_{\ell}^2$ iterations.
 \ENDFOR
 \end{algorithmic}
 \label{alg:nonpcGA}
}
\end{algorithm}

\section{Theoretical Analysis}
\label{sec:theory}

In this section, we prove mathematical runtime guarantees for our smart-restart~cGA.

\subsection{A General Performance Guarantee}

We follow the general approach of~\cite{Doerr21cgajump} of assuming that the runtime increases linearly with the population size from a given minimum size $\tilde \mu$ on.

\textbf{Assumption (L):} Let $p \in (0,1)$. Consider using the cGA with population size~$\mu$ to maximize a given function $f$. Assume that there are unknown $\tilde{\mu}$ and $T$ such that the cGA for all population sizes $\mu \ge \tilde{\mu}$ optimizes $f$ in $\mu T$ fitness evaluations with probability at least~$p$.

This Assumption~(L) is identical to the assumption taken in~\cite{Doerr21cgajump} except that there $p$ was required to be at least $3/4$, whereas we allow a general positive~$p$. Since most existing runtime analyses give bounds with success probability $1-o(1)$, this difference is, of course, not very important. We note that the proof of the result in~\cite{Doerr21cgajump} requires $p$ to be at least $3/4$, but we also note that an elementary probability amplification argument (via independent restarts) allows to increase a given success probability, rendering the result of~\cite{Doerr21cgajump} again applicable to arbitrary positive $p$.

Under this Assumption~(L), we obtain the following result. We note that it is non-asymptotic, which later allows to easily obtain asymptotic results also for non-constant parameters. We note that when assuming $p$ and $U$ to be constants (which is very natural), then the bound becomes $O(\max\{b\tilde{\mu}^2,{T^2}/{b},\tilde{\mu}T\})$.

\begin{theorem}
Let  $U >1$ and $b > 0$. Consider using the smart-restart cGA with update factor $U$ and generation budget $B_{\ell}=b\mu_{\ell}^2, \ell=1,2,\dots$, optimizing a function $f$ satisfying Assumption~(L) with $p \in (1-\frac{1}{U^2},1)$. Then the expected time until the optimum of $f$ is generated is at most 
\[
\left(\frac{U^2}{U^2-1}+\frac{(1-p)U^2}{1-(1-p)U^2}\right)\max\left\{b\tilde{\mu}^2,\frac{T^2}{b}\right\}+\frac{pU}{1-(1-p)U}\tilde{\mu}T
\]
fitness evaluations.
\label{thm:nonpcGAwAssume}
\end{theorem}

\begin{proof}
Let $\ell' = \min\{\ell \mid 2U^{\ell-1} \ge \tilde{\mu}, B_{\ell} \ge 2U^{\ell-1} T\} $. Then it is not difficult to see that $2U^{\ell'-1} \le U\max\{\tilde{\mu},T/b\}$ and that for any $\ell \ge \ell'$, the population size $\mu_{\ell} := 2U^{\ell-1}$ satisfies $\mu_{\ell} \ge \tilde{\mu}$ and $B_{\ell} \ge \mu_{\ell} T$. Hence, according to the assumption, we know the cGA with such a $\mu_{\ell}$ optimizes $f$ with probability at least $p$ in time~$\mu_{\ell} T$. Now the expected time when the smart-restart cGA finds the optimum of $f$ is at most
\begin{align*}
\sum_{i=1}^{\ell'-1}B_i&{}+p\cdot 2U^{\ell'-1}T + \sum_{i=1}^{\infty}(1-p)^ip\left(\sum_{j=0}^{i-1} B_{\ell'+j} + 2U^{\ell'+i-1}T\right)\\
\le &{} \sum_{i=1}^{\ell'-1}B_i+p\cdot 2U^{\ell'-1}T + \sum_{i=1}^{\infty}(1-p)^ip\sum_{j=0}^{i-1} B_{\ell'+j} + 2pU^{\ell'-1}T\sum_{i=1}^{\infty}(1-p)^iU^{i}\\
= &{} \sum_{i=1}^{\ell'-1}B_i+2pU^{\ell'-1}T + \sum_{j=0}^{\infty}B_{\ell'+j}p \sum_{i=j+1}^{\infty} (1-p)^i + 2pU^{\ell'-1}T\frac{(1-p)U}{1-(1-p)U}\\
= &{} \sum_{i=1}^{\ell'-1}B_i+\sum_{j=0}^{\infty}B_{\ell'+j}p \frac{(1-p)^{j+1}}{1-(1-p)}+\frac{2pU^{\ell'-1}T}{1-(1-p)U}\\
= &{} \sum_{i=1}^{\ell'-1}B_i+\sum_{j=0}^{\infty}(1-p)^{j+1} B_{\ell'+j}+\frac{2pU^{\ell'-1}T}{1-(1-p)U},
\end{align*}
where the first equality uses $(1-p)U \in (0,1)$ from $p \in (1-\frac{1}{U^2},1)$. With $B_{\ell}=b\mu_{\ell}^2=b(2U^{\ell-1})^2=4bU^{2\ell-2}$, we further compute
\begin{align*}
\sum_{i=1}^{\ell'-1}B_i{}&{}+\sum_{j=0}^{\infty}(1-p)^{j+1} B_{\ell'+j}+\frac{2pU^{\ell'-1}T}{1-(1-p)U}\\
= {}&{} \sum_{i=1}^{\ell'-1}4bU^{2i-2}+\sum_{j=0}^{\infty}(1-p)^{j+1} 4bU^{2\ell'+2j-2}+\frac{2pU^{\ell'-1}T}{1-(1-p)U}\\
= {}&{} \frac{4b(U^{2\ell'-2}-1)}{U^2-1} + \frac{4b(1-p)U^{2\ell'-2}}{1-(1-p)U^2}+\frac{2pU^{\ell'-1}T}{1-(1-p)U}\\
\le {}&{} \frac{bU^2\max\{\tilde{\mu}^2,T^2/b^2\}}{U^2-1} + \frac{b(1-p)U^2\max\{\tilde{\mu}^2,T^2/b^2\}}{1-(1-p)U^2}+\frac{pU\tilde{\mu}T}{1-(1-p)U}\\
= {}&{} \left(\frac{U^2}{U^2-1}+\frac{(1-p)U^2}{1-(1-p)U^2}\right)\max\left\{b\tilde{\mu}^2,\frac{T^2}{b}\right\}+\frac{pU}{1-(1-p)U}\tilde{\mu}T,
\end{align*}
where the second equality uses $(1-p)U^2 \in (0,1)$ from $p \in (1-\frac{1}{U^2},1)$ and the first inequality uses $2U^{\ell'-1} \le U\max\{\tilde{\mu}^2,T^2/b^2\}$.
\end{proof}

We recall that the complexity of the parallel-run cGA~\cite{Doerr21cgajump}, which is under the original assumption~\cite{Doerr21cgajump} but obviously also hold for our Assumption~(L). We formulate it in the following theorem.
\begin{theorem}\cite[Theorem~2]{Doerr21cgajump}
The expected number of fitness evaluations for the parallel-run cGA optimizing a function $f$ satisfying Assumption~(L) with $p \ge 3/4$ is ${O\left(\tilde{\mu}T\log(\tilde{\mu}T)\right)}$.
\label{thm:parallel}
\end{theorem}

Since the choice $b = \Theta(T / \tilde \mu)$ gives an asymptotic runtime of $O(\tilde \mu T)$ for the smart-restart cGA, we see that with the right choice of the parameters the smart-restart cGA can outperform the parallel-run cGA slightly. This shows that it indeed gains from its ability to abort unprofitable runs.

Our main motivation for regarding Assumption (L) was that this runtime behavior is often observed both in theoretical results (see, e.g., the survey~\cite{KrejcaW20}) and in experiments (see Section~\ref{sec:exper}). Unfortunately, some theoretical results were only proven under the additional assumption that $\mu$ is polynomially bounded in~$n$, that is, that $\mu = O(n^C)$ for some, possibly large, constant $C$. For most of these results, we are convinced that the restriction on $\mu$ is not necessary, but was only taken for convenience and in the light that super-polynomial values for $\mu$ would imply not very interesting super-polynomial runtimes. To extend such results to our smart-restart cGA in a formally correct manner, we now prove a version of Theorem~\ref{thm:nonpcGAwAssume} applying to such settings. More precisely, we regard the following assumption. 

\textbf{Assumption (L'):} Let $p \in (0,1)$. Consider using the cGA with population size~$\mu$ to maximize a given function $f$. Assume that there are unknown $\tilde{\mu}$, $\mu^+$,  and $T$ such that the cGA for all population sizes $\tilde \mu \le \mu \le \mu^+$ optimizes $f$ in $\mu T$ fitness evaluations with probability at least~$p$.

We prove the following result.

\begin{theorem}
Let  $U >1$ and $b > 0$. Consider using the smart-restart cGA with update factor $U$ and generation budget $B_{\ell}=b\mu_{\ell}^2, \ell=1,2,\dots$, optimizing a function $f$ satisfying Assumption~(L') with $p \in (1-\frac{1}{U^2},1)$. Let $\ell' = \min\{\ell \mid 2U^{\ell-1} \ge \tilde{\mu}, B_{\ell} \ge 2U^{\ell-1} T\}$ and $\calL := \{\ell \in \Z \mid \ell \ge \ell', 2 U^{\ell-1} \le \mu^+\}$. Then, apart from when an exceptional event of probability at most $(1-p)^{|\calL|}$ holds, the expected time until the optimum of $f$ is generated is at most 
\[
\left(\frac{U^2}{U^2-1}+\frac{(1-p)U^2}{1-(1-p)U^2}\right)\max\left\{b\tilde{\mu}^2,\frac{T^2}{b}\right\}+\frac{pU}{1-(1-p)U}\tilde{\mu}T
\]
fitness evaluations.
\label{thm:nonpcGAwAssume2}
\end{theorem}

\begin{proof}
  Let $A$ be the event that none of the runs of the cGA with population size $\mu_\ell = 2 U^{\ell-1}$ at most $\mu^+$ finds the optimum of $f$. As in the proof of Theorem~\ref{thm:nonpcGAwAssume}, each of the runs using populations size $\mu_\ell$, $\ell \in \calL$, with probability at least $p$ finds the optimum. Hence the event $A$ occurs with probability at most $(1-p)^{|\calL|}$. Conditional on the event $\neg A$, the smart-restart cGA behaves as if we would have Assumption (L). Hence when conditioning on $\neg A$, we can work under the assumption $(L)$. Now we note that the event $A$ contains the largest runtime estimates, namely $\mu T$ for certain $\mu > \mu^+$. Hence, when assuming (L), then conditioning on $\neg A$ can only reduce the expected runtime. Consequently, the runtime estimate of Theorem~\ref{thm:nonpcGAwAssume} is valid when conditioning on $\neg A$, and this both when working with assumption (L) and (L'). This proves the claim.
\end{proof}

\subsection{Specific Runtime Results}

The following examples show how to combine our general runtime analysis with known runtime results to obtain performance guarantees for the smart-restart cGA on several specific problems.
  
\subsubsection{\onemax and \jump}
We recall the runtime results of the cGA on  \onemax~\cite{SudholtW19} and \jump~\cite{Doerr21cgajump}.
\begin{theorem}\cite{SudholtW19,Doerr21cgajump}
Let $K>0$ be a sufficiently large constant and let $C>0$ be any constant. Consider the cGA with population size $K \sqrt n \ln n \le \mu \le n^C$. 
\begin{itemize}
\item The expected runtime on the \onemax function is $O(\mu \sqrt n)$ {\cite[Theorem~2]{SudholtW19}}.
\item With probability $1-o(1)$, the optimum of the \jump function with jump size $k< \tfrac{1}{20} \ln n$ is found in time $O(\mu \sqrt n)$ \cite[Theorem~9]{Doerr21cgajump}.
\end{itemize}
\label{thm:omjump}
\end{theorem}

We note that the \jump result also applies to \onemax simply because the \jump function with jump size $k=1$ has a fitness landscape that can in a monotonic manner be transformed into the one of the \onemax function. Hence for $n$ sufficiently large, we have Assumption~(L') satisfied with $\tilde \mu =  K \sqrt n \ln n$, $\mu^+ = n^C$, $T = O(\sqrt n)$, and $p = 1-o(1)$. Consequently, for any (constant) $U > 1$ we have $p\in (1-\frac{1}{U^2},1)$. Given the above information on $\tilde \mu$ and $T$, we see that any $b \in n^{-o(1)} \cap n^{o(1)}$ gives that $\ell' = \min\{\ell \mid 2U^{\ell-1} \ge \tilde{\mu}, B_{\ell} \ge 2U^{\ell-1} T\} = (1 \pm o(1)) \frac 12 \log_{U}(n)$. Since $\mu^+ = n^C$, we have $|\calL| = (1 \pm o(1)) (C - \frac 12) \log_{U}(n)$. With Theorem~\ref{thm:nonpcGAwAssume2}, we have the following result.

\begin{theorem}
Consider the smart-restart cGA with update factor $1 < U$ optimizing the \jump function with jump size $k< \tfrac{1}{20} \ln n$ or the \onemax function. Then, apart from a rare event of probability at most $n^{-\omega(1)}$, we have the following estimates for the expected runtime.
\begin{itemize}
\item If the budget factor $b$ is $\Theta(1/\log n)$, then the expected runtime is $O(n \log n)$.
\item If the budget factor $b$ is between $\Omega(1/\log^2 n)$ and $O(1)$, then the expected runtime is $O(n \log^2 n)$.
\end{itemize}
\label{thm:smartomjump}
\end{theorem}

Hence, our smart-restart cGA with $b=\Theta(1/\log n)$ has essentially the same time complexity as the original cGA (Theorem~\ref{thm:omjump}) with optimal population size. A constant $b$ results in  a slightly inferior runtime of $O(n \log^2 n)$, which is also the runtime guarantee for the parallel-run cGA (Theorem~\ref{thm:parallel}).

\subsubsection{Noisy \onemax}

For another example, we recall the runtime of the cGA (without artificial margins) on the \onemax function with additive centered Gaussian noise from~\cite{FriedrichKKS17}.

\begin{theorem}[{\cite[Theorem~5]{FriedrichKKS17}}]
Consider the $n$-dimensional \onemax function with additive centered Gaussian noise with variance $\sigma^2>0$. Then with probability $1-o(1)$, the cGA (without margins) with population size $\mu=\omega(\sigma^2\sqrt n \log n)$ has all frequencies at $1$ in $O(\mu \sigma^2\sqrt n \log (\mu n))$ iterations.
\label{thm:cgaFKKS}
\end{theorem}

We note that here the cGA is used without restricting the frequencies to the interval $[1/n,1-1/n]$, whereas more commonly (and in the remainder of this paper) the cGA is equipped with the margins $1/n$ and $1-1/n$ to avoid that frequencies reach the absorbing boundaries $0$ or $1$. Since our general runtime results do not rely on such implementation details but merely lift a result for a particular cGA to its smart-restart version, this poses no greater problems for us now. As a side remark, though, we note that we are very optimistic that the above result from~\cite{FriedrichKKS17} holds equally well for the setting with frequency margins. 

More interestingly, the runtime result above is not of the type that for $\mu$ sufficiently large, the expected runtime is $O(\mu T)$ for some $T$ (since $\mu$ appears also in the $\log(\mu n)$ term). Fortunately, with Theorem~\ref{thm:nonpcGAwAssume2} at hand, we have an easy solution. By only regarding values of $\mu$ that are at most $n^C$ for some constant $C$ (which we may choose), the $\log(\mu n)$ term can by bounded by $O(\log n)$. Since the minimal applicable $\mu$ (the $\tilde \mu$ in the notation of Theorem~\ref{thm:nonpcGAwAssume2}) depends on $\sigma^2$, this also implies that we can only regard polynomially bounded variances, but it is clear that any larger variances can be only of a purely academic interest. We thus formulate and prove the following result. We note that with more work, we could also have extended Theorem~\ref{thm:nonpcGAwAssume} to directly deal with the runtime behavior described in Theorem~\ref{thm:cgaFKKS}, mostly by exploiting that the geometric series showing up in the analysis do not change significantly when an extra logarithmic term is present, but this appears to be a lot of work for a logarithmic term for which it is not even clear if it is necessary in the original result.

\begin{theorem}
Let $C\ge 1$,  $U > 1$, and $h : [0,\infty) \to [0,\infty)$ in $\omega(1) \cap n^{o(1)}$. Consider the smart-restart cGA with the update factor $U$ and generation budget factor $b$ optimizing the $n$-dimensional \onemax function with additive centered Gaussian noise with variance $\sigma^2 \le n^C$. Then outside a rare event holding with probability $n^{-\omega(1)}$, the following runtime estimates are true.
\begin{itemize}
\item If $b=\Theta(1/h(n))$, then the expected runtime is $O(h(n) \sigma^4 n \log^2 n)$.
\item If $b = O(1) \cap \Omega(1/\log^2 n)$, then the expected runtime is $O(h(n) \sigma^4 n \log^3 n)$.
\end{itemize}
\label{thm:noisyom}
\end{theorem}

\begin{proof}
  By Theorem~\ref{thm:cgaFKKS}, we have Assumption (L') satisfied with $\tilde \mu = h(\mu) \sigma^2 \sqrt n \ln(n)$, $\mu^+ = n^{2C}$, $T = O(\sigma^2 \sqrt n \log n)$, and $p = 1-o(1)$. Consequently, any $b \in n^{-o(1)} \cap n^{o(1)}$ gives that $\ell' = \min\{\ell \mid 2U^{\ell-1} \ge \tilde{\mu}, B_{\ell} \ge 2U^{\ell-1} T\} = (1 \pm o(1)) (\frac 12 \log_{U}(n) + \log_{U}(\sigma^2)) \le (1+o(1)) C \log_{U}(n)$. Since $\mu^+ = n^{2C}$, we have $|\calL| \ge (1 \pm o(1)) C \log_{U}(n)$. With Theorem~\ref{thm:nonpcGAwAssume2}, we have proven our claim.
\end{proof}

We remark that the parallel-run cGA has an expected runtime of $O(h(n) \sigma^4 n \log^3 n)$ outside a rare event of probability $n^{-\omega(n)}$.

\section{Experimental Results}
\label{sec:exper}

In this section, we experimentally analyze the smart-restart cGA proposed in this work. Since such data is not available from previous works, we start with an investigation how the runtime of the original cGA depends on the population size $\mu$. This will in particular support the basic assumption underlying the smart-restart cGA (and the parallel-run cGA from~\cite{Doerr21cgajump}) that the runtime can be excessively large when $\mu$ is below some threshold, and moderate and linearly increasing with $\mu$ when $\mu$ is larger than this threshold.

Since the choice of the right population size is indeed critical for a good performance of the cGA, we then analyze the performance of the two existing approaches to automatically deal with the problem of choosing $\mu$. Our focus is on understanding how one can relieve the user of an EDA from the difficult task of setting this parameter, not on finding the most efficient algorithm for the benchmark problems we regard. For this reason, we do not include other algorithms in this investigation. 

\subsection{Test Problems}

Based on the above goals, we selected the four benchmark functions \onemax, \LO, \jump, and \DLB as optimization problems. For most of them also some mathematical runtime analyses exist, which help to understand and interpret the experimental results.

All four problems are defined on binary representations (bit strings) and we use $n$ to denote their length. The \textbf{\onemax} problem is one of the easiest benchmark problems. The \onemax fitness of a bit string is simply the number of ones in the bit string. Having the perfect fitness-distance correlation, most evolutionary algorithms find it easy to optimize \onemax, a common runtime is $\Theta(n \log n)$. Also, mathematical runtime analyses are aided by its simple structure (see, e.g.,~\cite{Muhlenbein92,GarnierKS99,JansenJW05,Witt06,RoweS14,DoerrK15,AntipovDFH18}), though apparently for EDAs the runtime of \onemax is highly non-trivial. The known results for EDAs are the following. The first mathematical runtime analysis for EDAs by Droste~\cite{Droste06} together with the recent work~\cite{SudholtW19} shows that the cGA can efficiently optimize \onemax in time $\Theta(\mu\sqrt n)$ when $\mu \ge K \sqrt n \ln(n)$ for some sufficiently large constant $K$. As the proofs of this result show (and the same could be concluded from the general result~\cite{DoerrZ20tec}), in this parameter regime there is little genetic drift. Throughout the runtime, with high probability, all bit frequencies stay above $\frac 14$. For hypothetical population sizes below the $\sqrt n \log n$ threshold, the situation is less understood. However, the lower bound of $\Omega(\mu^{1/3} n)$ valid for all $\mu = O\left(\frac{\sqrt n}{\ln(n) \ln\ln(n)}\right)$ proven in~\cite{LenglerSW18} together with its proof shows that in this regime the cGA suffers from genetic drift, leading to (mildly) higher runtimes.

The \textbf{\LO} benchmark is still an easy unimodular problem, however, typically harder than \onemax. The \leadingones value of a bit string is the number of ones in it, counted from left to right, until the first zero. How simple randomized search heuristics optimize \leadingones is extremely well understood~\cite{DrosteJW02,JansenJW05,Witt06,BottcherDN10,Sudholt13,Doerr19tcs,DoerrDL19,LissovoiOW20}, many EAs optimize this benchmark in time $\Theta(n^2)$. Surprisingly, no theoretical results are known on how the cGA optimizes \leadingones. However, the runtime of another EDA, the UMDA, with population sizes $\mu = \Theta(\lambda)$ with suitable implicit constants and $\lambda = \Omega(\log n)$ was shown to be $O(n \lambda \log(\lambda) + n^2)$~\cite{DangL15} and, recently, $\Theta(n \lambda)$ for $\lambda = \Omega(n \log n)$~\cite{DoerrK21tcs}. Without going into details on this EDA not discussed so far in this work, we remark that~\cite{DoerrZ20tec} for this situation shows that genetic drift occurs when $\lambda$ is below a threshold of $\Theta(n)$. Consequently, these results show a roughly linear influence of $\lambda$ on the runtime when $\lambda$ is (roughly) at least linear in $n$, but below this value, there is apparently no big penalty for running the EDA in the genetic drift regime. For the cGA, we will observe a different behavior, which also indicates that translating general behaviors from one EDA to another, even within the class of univariate EDAs, has to be done with caution.

The \textbf{\jump} benchmark is a class of multimodal fitness landscapes of scalable difficulty. For a difficulty parameter $k$, the fitness landscape is isomorphic to the one of \onemax except that there is a valley of low fitness of width $k$ around the optimum. More precisely, all search points in distance $1$ to $k-1$ from the optimum have a fitness lower than all other search points. Recent results~\cite{HasenohrlS18,Doerr21cgajump} show that when $\mu$ is large enough (so that the genetic drift is low, that is, all bit frequencies stay above $\frac 14$), then the cGA can optimize \jump functions quite efficiently and significantly more efficient than many classic evolutionary algorithms. We omit some details and only mention that for $k$ not too small, a runtime exponential in $k$ results from a population size $\mu$ that is also exponential in $k$. This is much better than the $\Omega(n^k)$ runtime of typical mutation-based evolutionary algorithms~\cite{DrosteJW02,DoerrLMN17,RajabiW20,Doerr20gecco} or the $n^{O(k)}$ runtime bounds shown for several crossover-based algorithms~\cite{DangFKKLOSS16,AntipovDK20} (we note that $O(n)$ and $O(n \log n)$ runtimes have been shown in \cite{WhitleyVHM18,RoweA19}, however, these algorithms appear quite problem-specific and have not been regarded in other contexts so far). It was not known whether the runtime of the cGA becomes worse in the regime with genetic drift, but our experimental results now show this. 

The \textbf{\DLB} benchmark was introduced in~\cite{LehreN19foga}. It can be seen as a deceptive version of the \leadingones benchmark. In \DLB, the bits are partitioned into blocks of length two in a left-to-right fashion. The fitness is computed as follows. Counting from left to right, each block that consists of two ones contributes two to the fitness, until the first block is reached that does not consist of two ones. This block contributes one to the fitness if it consists of two zeros, otherwise it contributes zero. All further blocks do not contribute to the fitness. The main result in~\cite{LehreN19foga} is that when $\mu = \Theta(\lambda)$ and $\lambda = o(n)$, the runtime of the UMDA on \DLB is exponential in $\lambda$. With $\lambda$ as small as $o(n)$ and a runtime that is at least quadratic, this result lies in a regime with strong genetic drift according to~\cite{DoerrZ20tec}. When $\lambda = \Omega(n \log n)$, a runtime of approximately $\frac 12 \lambda n$ was shown in~\cite{DoerrK20evocop}. Hence for this function and the UMDA as optimizer, the choice of the population size is again very important. This was the reason for including this function into our set of test problems and the results indicate that indeed the cGA shows a behavior similar to what the mathematical results showed for the UMDA. Other runtime results on the \DLB function include several $O(n^3)$ runtime guarantees for classic EAs~\cite{LehreN19foga} as well as a $\Theta(n^2)$ runtime for the Metropolis algorithm and an $O(n \log n)$ runtime guarantee for the significance-based cGA~\cite{WangZD21}.

\subsection{Experimental Settings}

We ran the original cGA (with varying population sizes), the parallel-run cGA, and our smart-restart cGA (with two generation budget factors) on each of the above-described four problems, both in the classic scenario without noise and in the presence of Gaussian posterior noise of four different strengths. For each experiment for the parallel-run cGA and our smart-restart cGA, we conducted 20 independent trials. For reasons of extremely large runtimes in the regime with genetic drift, only 10 independent trials were conducted for the original cGA for all tested population sizes. The detailed settings for our experiments were as follows.
\begin{itemize}
\item Benchmark functions: \onemax (problem size $n=100$), \LO ($n=50$), \jump ($n=50$ and the jump size $k=10$), and \DLB ($n=30$).
\item Noise model: additive centered Gaussian posterior noise with variances $\sigma^2=\{0,n/2,n,2n,4n\}$. See Section~\ref{ssec:intronoise} for a detailed description.
\item Since the original cGA with unsuitable population sizes did not find the optimum in reasonable time, we imposed the following maximum numbers of generations and aborted the run after this number of generations: $n^5$ for \onemax and \LO, $n^{k/2}$ for \jump, and $10n^5$ for \DLB. We did not define such a termination criterion for the parameter-less versions of the cGA since they always found the optimum in an affordable time.
\item Population size of the original cGA: $\mu=2^{[5..10]}$ for \onemax, $\mu=2^{[2..10]}$ for \LO, $\mu=2^{[9..18]}$ for \jump, and $\mu=2^{[1..14]}$ for \DLB. 
\item Generation budget factor $b$ for the smart-restart cGA: $8$ and $0.5/\ln n$. As explained in the introduction, the generation budget factors $b=8$ and $\Theta(1 / \ln n)$ are two proper choices. We chose the constant $0.5$ based on the experimental results on \jump and \DLB without noise (noise variance $\sigma^2=0$) in Figures~\ref{fig:jump} and~\ref{fig:dlb}. To avoid an overfitting, we ignored all other experiments when choosing the constant. 
\item Update factor $U$ for the smart-restart cGA: $2$. Doubling the parameter value after each unsuccessful run ($U=2$) is a natural choice~\cite{Doerr21cgajump} for a sequential parameter search. 

\end{itemize}

\subsection{Experimental Results and Analysis I: The cGA with Different Population Sizes}
 
Figures~\ref{fig:om}-\ref{fig:dlb} (the curves except the last three items on the $x$-axis in each figure) show the runtime (measured by the number of fitness evaluations) of the original cGA with different population sizes when optimizing our four test functions under Gaussian noise with different variances (including the noise-free setting $\sigma^2=0$). 

\begin{figure}[!ht]
\centering
\includegraphics[width=5.0in]{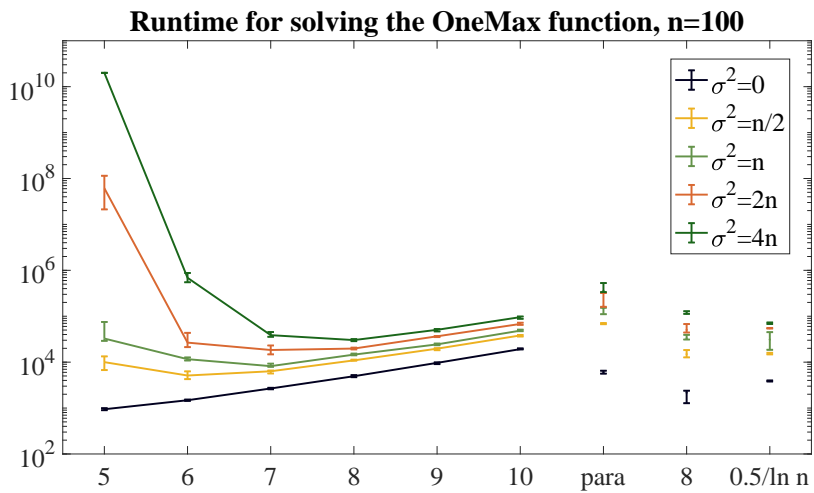}
\caption{The median number of fitness evaluations (with the first and third quartiles) of the original cGA with different $\mu$ ($\log_2 \mu \in \{5,6,\dots,10\}$), the parallel-run cGA (``para''), and the smart-restart cGA with two budget factors ($b=8$ and $b=0.5/\ln n$) on the \onemax function ($n=100$) under Gaussian noise with variances $\sigma^2=0, n/2, n, 2n, 4n$ in 20 independent runs (10 runs for the original cGA). }
\label{fig:om}
\end{figure} 

\begin{figure}[!ht]
\centering
\includegraphics[width=5.0in]{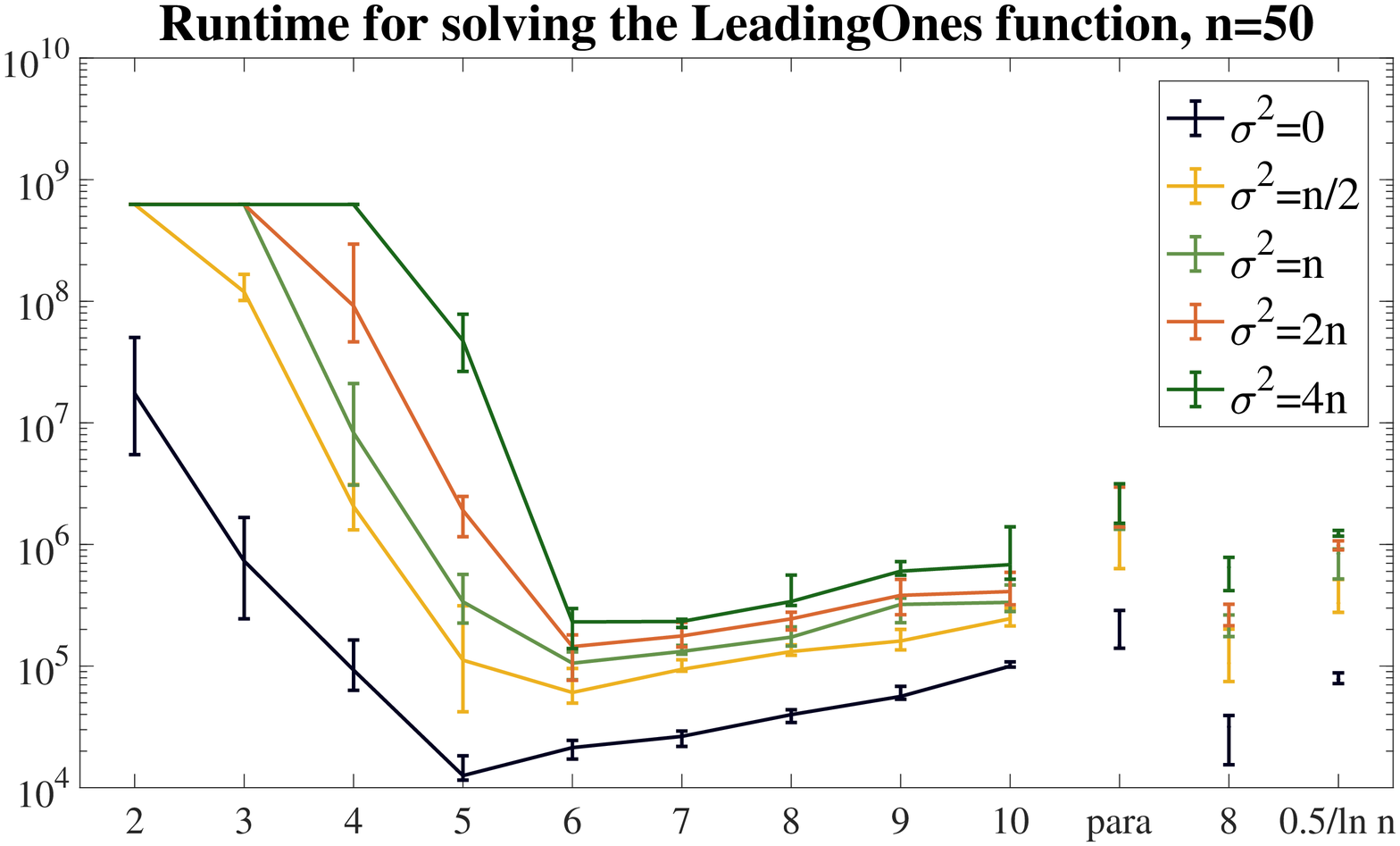}
\caption{The median number of fitness evaluations (with the first and third quartiles) of the original cGA with different $\mu$ ($\log_2 \mu \in \{2,3,\dots,10\}$), the parallel-run cGA (``para''), and the smart-restart cGA with two budget factors ($b=8$ and $b=0.5/\ln n$) on the \LO function ($n=50$) under Gaussian noise with variances $\sigma^2=0, n/2, n, 2n, 4n$ in 20 independent runs (10 runs for the original cGA). }
\label{fig:lo}
\end{figure} 

\begin{figure}[!ht]
\centering
\includegraphics[width=5.0in]{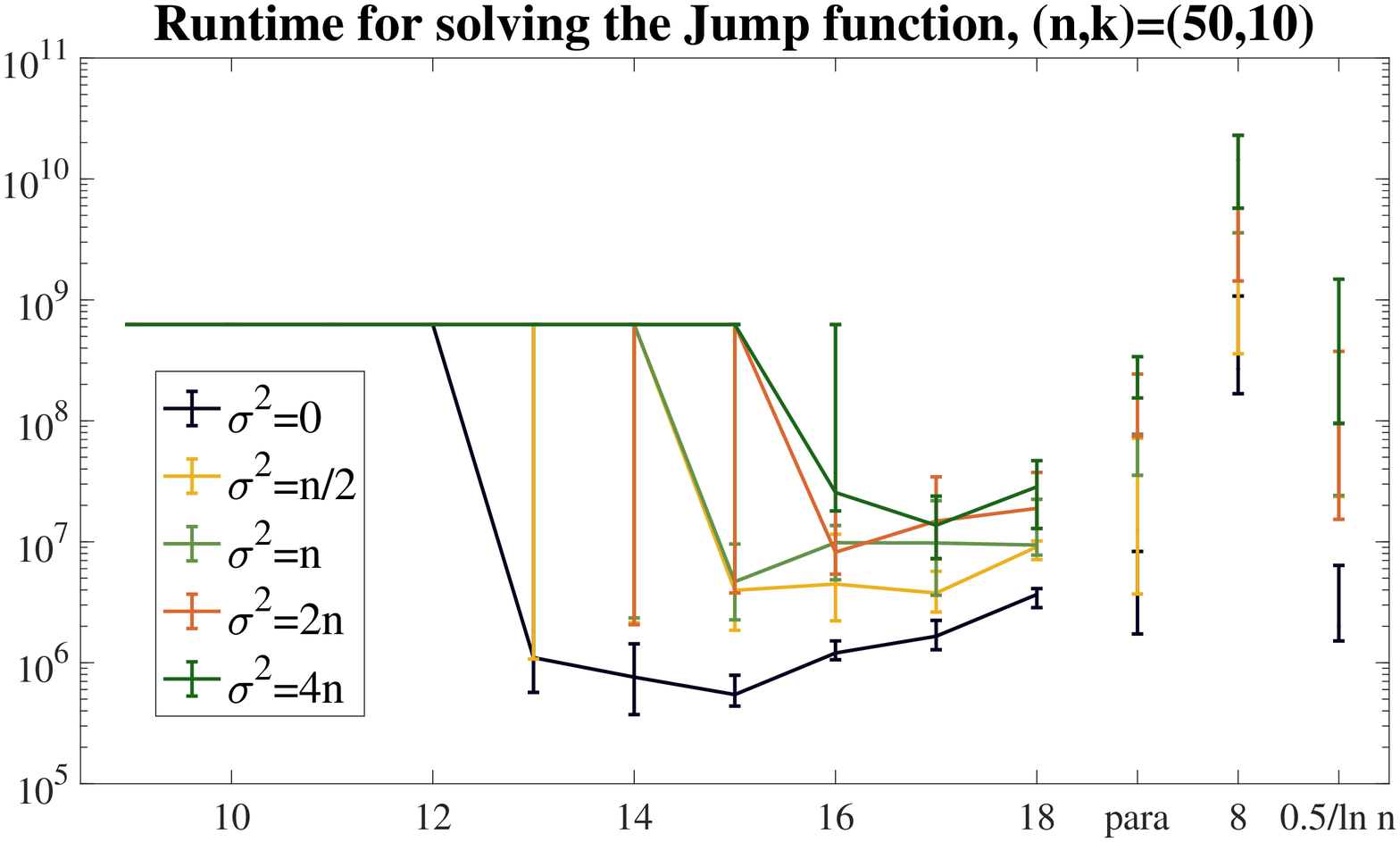}
\caption{The median number of fitness evaluations (with the first and third quartiles) of the original cGA with different $\mu$ ($\log_2 \mu \in \{9,10,\dots,18\}$), the parallel-run cGA (``para''), and the smart-restart cGA with two budget factors ($b=8$ and $b=0.5/\ln n$) on the \jump function with $(n,k)=(50,10)$ under Gaussian noise with variances $\sigma^2=0, n/2, n, 2n, 4n$ in 20 independent runs (10 runs for the original cGA). }
\label{fig:jump}
\end{figure} 

\begin{figure}[!ht]
\centering
\includegraphics[width=5.0in]{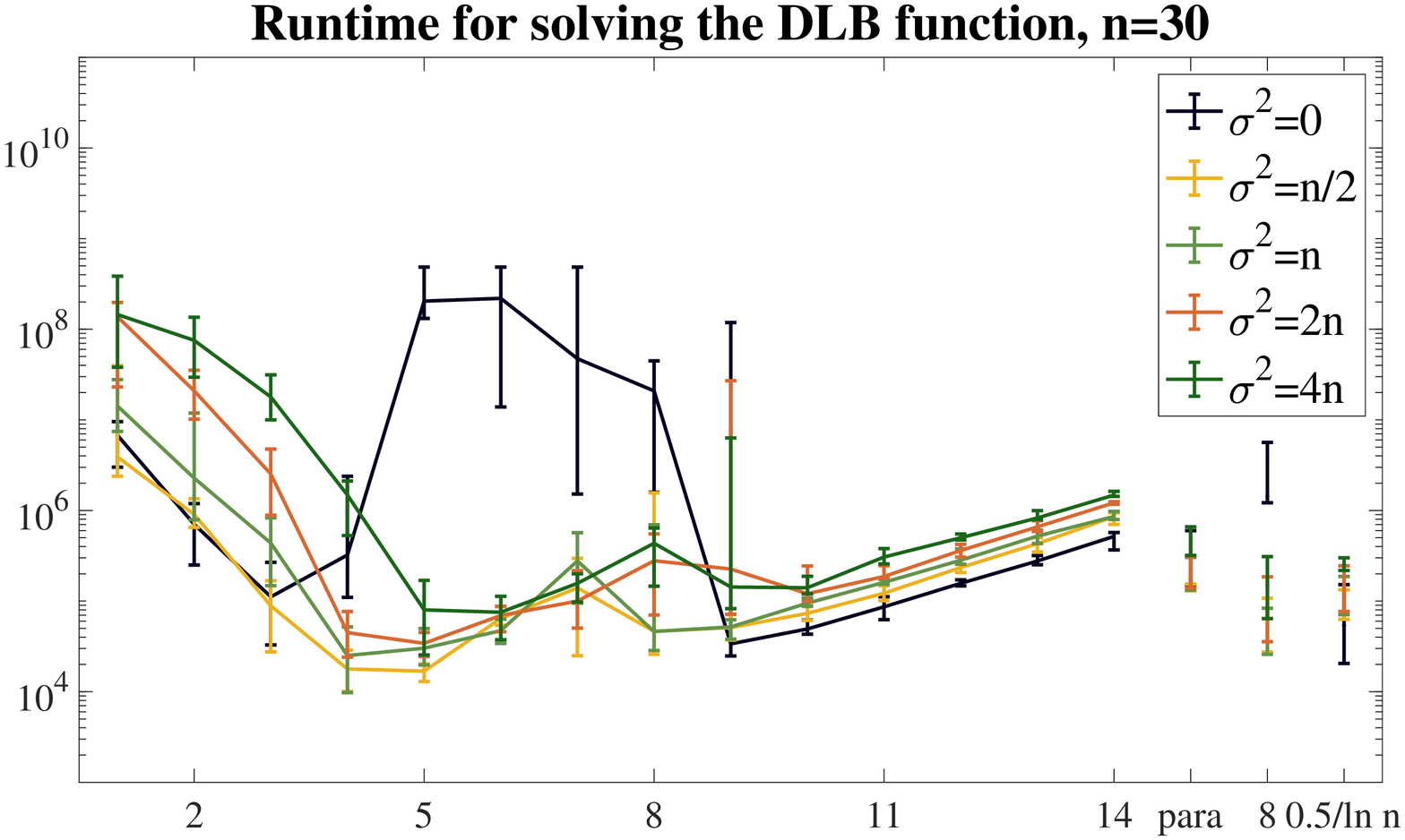}
\caption{The median number of fitness evaluations (with the first and third quartiles) of the original cGA with different $\mu$ ($\log_2 \mu \in \{1,2,\dots,14\}$), the parallel-run cGA (``para''), and the smart-restart cGA with two budget factors ($b=8$ and $b=0.5/\ln n$) on the \DLB function ($n=30$) under Gaussian noise with variances $\sigma^2=0, n/2, n, 2n, 4n$ in 20 independent runs (10 runs for the original cGA). }
\label{fig:dlb}
\end{figure}


The results displayed in Figures~\ref{fig:om}--\ref{fig:dlb} typically show that the runtime of the cGA is large both for small values of $\mu$ and for large values (exceptions are the noise-free runs on \onemax (no increase for small values for our parameters -- such an increase was observed in the $n=500$ experiments in~\cite{DoerrZ20gecco}, but only for $\mu=2$) and on \DLB (showing a bimodal runtime pattern)). For small values of $\mu$, the runtime increase is often steep and accompagnied by larger variances of the runtime. For large values of $\mu$, we typically observe a moderate, roughly linear increase of the runtime. The variances are relatively small here. 

We note that for runs that were stopped because the maximum number of generations was reached, we simply and bluntly counted this maximum number of generations as runtime. Clearly, there are better ways to handle such incomplete runs, but since a fair computation for these inefficient parameter ranges is not too important, we did not start a more elaborate evaluation. 

Let us regard the increase of the runtime for smaller population sizes in more detail. For \onemax, this increase is not very pronounced except for large values of $\sigma^2$. This fits to the known theoretical results which show that also in regimes with strong genetic drift, the cGA can optimize efficiently by, very roughly speaking, imitating the  optimization behavior of the \oea. Nevertheless, there is more profitable middle range for $\mu$ and the $\mu$ values of this range increase with increasing noise levels (which is equivalent to say with increasing runtimes). This suggests that also on easy fitness landspaces such as the one of \onemax, genetic drift can lead to performance losses.

For \leadingones, we observe a clearer optimal value for $\mu$ (depending on the noise level). Reducing $\mu$ leads to a clear (and drastic for large $\sigma$) increase of the runtime, typically at least by a factor of $10$ for each halving of $\mu$. The runtime distributions are less concentrated than for large values of $\mu$, but overall we still see a relatively regular runtime behavior. 

For \jump functions, the optimal $\mu$-value is not as clearly visible as for \leadingones, however reducing the population size $\mu$ below the efficient values leads to a catastrophic increase of the runtime, often leading to the runs stopped because the maximum number of generations is reached. We also observe a drastic increase of the variance of the runtime here. This indicates that some runs were very lucky to not suffer from genetic drift and then finished early (at a runtime as if the linear regimes was continued), whereas others suffered from genetic drift and thus took very long or never finished within the time limit. We note that when some frequencies reach the lower boundary of the frequency range due to  genetic drift, then it takes relatively long to move them back into the middle regime (simply because of the low sampling variance when the frequency is at a boundary value). During this longer runtime, of course, the remaining frequencies are still prone to genetic drift (recall that the quantitative analysis~\cite{DoerrZ20tec} shows that genetic drift is more likely the longer the run takes). These mathematical considerations and the experimental results indicate that for objective functions which could suffer from genetic drift, there are two very distinct extremal regimes: either no frequency reaches the wrong boundary and the optimization is efficient, or many frequencies reach the wrong boundary and the optimization is highly inefficient. 

The runtime behavior on \DLB is harder to understand. There is a clear ``linear regime'' from roughly $\mu=2^9$ on (again with very small variances). There is also a steep increase of the runtimes typically when lowering $\mu$ below $2^4$. In between these two regimes, the runtime behavior is hard to understand. The noisy runs show a small increase of the runtime in this middle regime together with slighly increased variances. The noise-free runs, however, are massively slower than the noisy ones, with large variances and a decent number of unsuccessful runs. We have no explanation for this. 

Apart from the runtimes on \DLB (though to some extent also here, namely in the noisy runs), our results indicate a runtime behavior as described in Assumption~(L). 
%
As a side result, this data confirms that the cGA has a good performance on noise-free \jump functions, not only in asymptotic terms as proven in~\cite{HasenohrlS18,Doerr21cgajump}, but also in terms of actual runtimes for concrete problem sizes. On a \jump function with parameters $n=50$ and $k=10$, a classic mutation-based algorithm would run into the local optimum and from there would need to generate the global optimum via one mutation. For standard bit mutation with mutation rate $\frac 1n$, this last step would take an expected time of $n^{k} (\frac {n}{n-1})^{n-k}$, which for our values of $n$ and $k$ is approximately $2.2 \cdot 10^{17}$. With the asymptotically optimal mutation rate of $\frac kn$ determined in~\cite{DoerrLMN17}, this time would still be approximately $7.3 \cdot 10^{10}$. In contrast, the median optimization time of the cGA with $\mu \in 2^{[15..18]}$ is always below $4 \cdot 10^{6}$.

Our data also indicates that a good performance of the cGA can often be obtained with much smaller population sizes (and thus more efficiently) than what previous theoretical works suggest. For example, in \cite{FriedrichKKS17} a population size of $\omega(\sigma^2 \sqrt n \log n)$ was required for the optimization of a noisy \onemax function via the cGA. In their experiments on a noisy \onemax function with $n=100$ and ${\sigma^2=n}$, a population size (called $K$ in~\cite{FriedrichKKS17} to be consistent with previous works) of $\mu = 7\sigma^2 \sqrt n (\ln n)^2\approx 148{,}000$ was used, which led to a runtime of approximately $200{,}000$ (data point for $\sigma^2 = 100$ interpolated from the two existing data points for $\sigma^2 = 64$ and $\sigma^2 = 128$ in the left chart of Figure~1 in~\cite{FriedrichKKS17}).\footnote{We have to admit that we cannot fully understand this number of $200{,}000$ and expect that it should be much larger. Our skepticism is based both on theoretical and experimental considerations. On the theoretical side, we note that even in the absence of noise and with the frequency vector having the (for this purpose) ideal value $\tau = (\frac 12, \dots, \frac 12)$, the sum $\|\tau\|_1$ of the frequency values increases by an expected value of $O(\frac 1\mu \sqrt n)$ only (with small leading constant; an absolute upper bound of $\frac 1{2\mu} \sqrt n$ follows, e.g., easily from~\cite{BerendK13}). Hence after only $200{,}000$ iterations, the frequency sum $\|\tau\|_1$ should still be relatively close to $n/2$. Since the probability to sample the optimum is $\prod_{i=1}^n (1-\tau_i) \le \exp(-\|\tau\|_1)$, it appears unlikely that the optimum is sampled within that short time. Our experimental data displayed in Figure~\ref{fig:om} suggests an affine-linear dependence of the runtime on $\mu$ when $\mu$ is at least $2^8$. From the median runtimes for $\mu=2^9$ and $\mu = 2^{10}$, which are $T_9 = 24{,}384$ and $T_{10} = 48{,}562$, we would thus estimate a runtime of $T(\mu) = T_9 + (T_{10}-T_9) (\mu - 2^9) 2^{-9}$ for $\mu \ge 2^{10}$, in particular, $T(7\sigma^2 \sqrt n (\ln n)^2) = 7{,}010{,}551$ for the data point $\sigma^2 = 100$ and $n=100$. 
To resolve this discrepancy, we conducted $20$ runs of the cGA with $\mu = \lfloor 7\sigma^2 \sqrt n (\ln n)^2 + \frac 12 \rfloor$, $\sigma^2 = 100$, $n = 100$ and observed a median runtime of 5,728,969 (and a low variance, in fact, all 20 runtimes were in the interval $[5{,}042{,}714; 6{,}131{,}522]$).} In contrast, our experiments displayed in Figure~\ref{fig:om} suggest that population sizes between 64 and 256 are already well sufficient and give runtimes clearly below $20{,}000$.

\subsection{Experimental Results and Analysis II: Runtimes of the Parallel-run cGA and the Smart-restart cGA}

The three right-most items on the $x$-axis in Figures~\ref{fig:om}--\ref{fig:dlb} show the runtimes of the parallel-run cGA and the smart-restart cGA (with two generation budget factors~$b$). We see that for the easy functions \onemax and \LO under all noise assumptions, the smart-restart cGA with both values of $b$ has a smaller runtime than the parallel-run cGA. This can be explained from the runtime data of the original cGA in the corresponding figures: Since the runtimes are similar for several population sizes, the parallel-run cGA with its strategy to assign a similar budget to different population sizes wastes computational power, which the smart-restart cGA saves by aborting some processes early. For both functions, the larger generation budget typically is superior. This fits again to the data on the original cGA and to our interpretation that genetic drift here is not so detrimental. Consequently, it is less to gain from aborting a run and starting a new one with a different population size.

More interesting are the results for \jump and \DLB. We recall that here a wrong choice of the population size can be catastrophic, so these are the two functions where not having to choose the population size is a big advantage for the user. What is clearly visible from the data is that here the smaller generation budget is preferable for the smart-restart cGA. This fits to our previously gained intuition that for these two functions, genetic drift is detrimental. Hence there is no gain from continuing a run that is suffering from genetic drift (we note that there is no way to detect genetic drift on the fly -- a frequency can be at a (wrong) boundary value due to genetic drift or at a (correct) boundary value because of a sufficiently strong fitness signal). 

What is clear as a general rule is that both algorithms, the parallel-run cGA and the smart-restart cGA with the small generation budget, clearly do a good job in successfully running the cGA with a reasonable population size -- recall that for both of the difficult functions, a wrong choice of the population size can easily imply that the cGA does not find the optimum in $10^8$ iterations.

\section{Conclusion}
\label{sec:con}
Choosing the right population size for estimation-of-distribution algorithms is one of the key difficulties for their practical usage. In order to remove the population size as a parameter and thus make the EDA easier to use, this paper proposed a parameter-less framework for EDAs, using the compact genetic algorithm as example. This framework is a simple restart strategy with exponentially growing population size, but different from previous works it sets a prior generation budget for each population size based on a recent quantitative analysis estimating when genetic drift is likely to occur and render the EDA inefficient. 

Under a reasonable assumption on how the runtime depends on the population size, we theoretically analyzed our scheme and observed that it can lead to asymptotically optimal runtimes for the cGA. 

Via extensive experiments on \onemax, \LO, \jump, and \DLB, we showed the efficiency of the parameter-less cGA, also when compared with the parallel-run cGA. The results for the original cGA with different population sizes experimentally show that the population size is crucial for the performance of the cGA and that the theoretically suggested population size can be far away from the right one.

We positively believe that our parameter-less framework for the cGA can be also applied to other univariate EDAs, again building on the quantitative analysis of genetic drift in~\cite{DoerrZ20tec}. The problem of how to cope with genetic drift, naturally, is equally interesting for multivariate EDAs. For these, however, our theoretical understanding is limited to very few results such as~\cite{ZhangM04,LehreN19foga,DoerrK20gecco}. In particular, a quantitative understanding of genetic drift comparable to~\cite{DoerrZ20tec} is completely missing. Another interesting question is if dynamic choices of the population size in EDAs can be fruitful. In classic EAs, dynamic parameter choices have recently been used very successfully to overcome the difficulty of finding a suitable static parameter value, see, e.g., the survey~\cite{DoerrD20bookchapter}. How to use such ideas for EDAs is currently not at all clear. 

\section*{Acknowlegment}
This work was supported by a public grant as part of the Investissement d'avenir project, reference ANR-11-LABX-0056-LMH, LabEx LMH.

This work was also supported by Guangdong Provincial Key Laboratory (Grant No. 2020B121201001), the Program for Guangdong Introducing Innovative and Enterpreneurial Teams (Grant No. 2017ZT07X386); Guangdong Basic and Applied Basic Research Foundation (Grant No. 2019A1515110177);  Shenzhen Peacock Plan (Grant No. KQTD2016112514355531); and the Program for University Key Laboratory of Guangdong Province (Grant No. 2017KSYS008). 

\newcommand{\etalchar}[1]{$^{#1}$}

}
\end{document}